\documentclass[twocolumn,abstract=true,10pt]{scrartcl}

\usepackage[T1]{fontenc}
\usepackage[utf8]{inputenc}
\usepackage{textcomp}
\usepackage{mathptmx}
\DeclareMathAlphabet{\mathcal}{OMS}{cmsy}{m}{n}
\usepackage{amsmath,amsthm,amsfonts,amssymb}
\usepackage{mathtools}
\usepackage[super]{natbib}
\usepackage{graphicx}
\usepackage{todonotes}
\usepackage[left=12mm,right=12mm,top=20mm]{geometry}
\usepackage{hyperref}

\hypersetup{
  pdfauthor={Tudor-Dan Rad, Daniel F. Scharler, Hans-Peter Schröcker},%
  pdftitle={The Kinematic Image of RR, PR, and RP Dyads},%
  pdfkeywords={kinematic map, dual quaternion, Study quadric, null cone, revolute joint, prismatic joint, fiber projectivity, vertical Darboux motion},%
  hidelinks=true,%
}

\newtheorem{theorem}{Theorem}
\newtheorem{lemma}{Lemma}
\newtheorem{corollary}{Corollary}
\theoremstyle{definition}
\newtheorem{definition}{Definition}
\theoremstyle{remark}
\newtheorem{remark}{Remark}
\newtheorem{example}{Example}

\newcommand{\eps}{\varepsilon}
\newcommand{\D}{\mathbb{D}}
\renewcommand{\H}{\mathbb{H}}
\newcommand{\C}{\mathbb{C}}
\newcommand{\R}{\mathbb{R}}
\newcommand{\qi}{\mathbf{i}}
\newcommand{\qj}{\mathbf{j}}
\newcommand{\qk}{\mathbf{k}}
\newcommand{\ci}{\mathrm{i}}
\newcommand{\cj}[1]{\overline{#1}}
\newcommand{\Norm}[1]{\Vert{#1}\Vert}
\newcommand{\quadric}[1]{\mathcal{#1}}
\newcommand{\QQ}{\quadric{Q}}
\newcommand{\SQ}{\quadric{S}}
\newcommand{\identity}{I}
\newcommand{\EG}{[\eps\H]}
\newcommand{\DG}{[\H]}
\newcommand{\TO}{\tau_1}
\newcommand{\TS}{[T]}
\newcommand{\EQ}{\quadric{Y}}
\newcommand{\NC}{\quadric{N}}
\newcommand{\EE}{\quadric{E}}
\newcommand{\SE}[1][3]{\mathrm{SE}(#1)}
\newcommand{\qf}{\omega}
\newcommand{\fiberproj}{\varphi}
\newcommand{\rM}[1]{\varrho(#1)}
\newcommand{\lM}[1]{\lambda(#1)}
\newcommand{\RM}[1]{R(#1)}
\newcommand{\LM}[1]{L(#1)}
\newcommand{\tp}{\intercal}
\newcommand{\zeromat}{O}

\setkomafont{title}{\rmfamily\bfseries}
\setkomafont{sectioning}{\rmfamily\bfseries}

\title{The Kinematic Image of RR, PR, and RP Dyads}
\author{Tudor-Dan Rad, Daniel F. Scharler, Hans-Peter Schröcker\\
  Unit Geometry and CAD, University of Innsbruck}

\begin{document}

\maketitle

\begin{abstract}
  We provide necessary and sufficient conditions for admissible
  transformations in the projectivised dual quaternion model of rigid
  body displacements and we characterise constraint varieties of dyads
  with revolute and prismatic joints in this model. Projective
  transformations induced by coordinate changes in moving and/or fixed
  frame fix the quadrics of a pencil and preserve the two families of
  rulings of an exceptional three-dimensional quadric. The constraint
  variety of a dyad with two revolute joints is a regular ruled
  quadric in a three-space that contains a ``null quadrilateral''. If
  a revolute joint is replaced by a prismatic joint, this
  quadrilateral collapses into a pair of conjugate complex null lines
  and a real line but these properties are not sufficient to
  characterise such dyads. We provide a complete characterisation by
  introducing a new invariant, the ``fiber projectivity'', and we
  present examples that demonstrate its potential to explain hitherto
  not sufficiently well understood phenomena.
\end{abstract}


\section{Introduction}
\label{sec:introduction}

A common technique in theoretical and applied kinematics is the use of
a point model for the group $\SE$ of rigid body displacements. One
prominent example is the projectivised dual quaternion model of $\SE$
which exhibits particularly nice geometric and algebraic properties
\cite{bottema90,selig05,husty12,klawitter15}. In this article, we
revisit some fundamental concepts related to this model, the
transformation group generated by coordinate changes in the moving and
the fixed frame and the kinematic images of dyads with revolute and
prismatic joints. While numerous necessary conditions on these objects
are well-known, we contribute sets of provably sufficient conditions.

Our characterisation of the transformation group generated by
coordinate changes in fixed and moving frame
(\autoref{sec:transformation-group}) is based on the pencil of
quadrics spanned by the Study quadric and null cone. These are
quadrics corresponding to dual quaternions of real norm and purely
dual norm, respectively. Admissible transformations fix each member of
this pencil and, in addition, preserve the two families of rulings on
a further quadric in a subspace of dimension three. This leads to the
important distinction between ``left'' and ``right'' rulings. At this
point we also introduce a further invariant, the ``fiber
projectivity'', which will be crucial in our latter characterisation
of dyads with prismatic joints in \autoref{sec:characterisation2}.

The relative position of two rigid bodies can be constrained by a link
or a sequence of links. Fixing one of the two bodies, the collection
of all possible poses (position and orientation) of the other is
called a \emph{constrained variety.} These are important objects in
the study of open and closed kinematic chains, in linkage synthesis or
analysis and other fields. In \autoref{sec:characterisation}, we
characterise constraint varieties generated by dyads of two revolute
joints (``RR dyads'') as regular ruled quadrics in the Study quadric
that contain four complex rulings of the null cone.

In \autoref{sec:characterisation2} we extend this result to dyads
containing one prismatic and one revolute joint (``RP dyads'' and ``PR
dyads''). It is tempting to view them as limiting cases of RR dyads
where one joint axis becomes ``infinite'' (lies in the plane at
infinity). Indeed, their kinematic image is a regular ruled quadric in
the Study quadric that intersects the null cone in two complex lines
and a real transversal line. Nonetheless, this viewpoint is not
complete because of the possibility of commuting R and P joints
(``cylindrical joints''). A closer investigation leads us to a more
refined concept involving the fiber projectivity which allows to
distinguish between the RP, the PR, and the cylindrical case.

We conclude this paper with an application of our results to a
recently presented non-injective extension of the classical kinematic
map \cite{pfurner16}. Here, commuting RP dyads appear naturally as
kinematic images of straight lines. We use this to prove that the
extended kinematic image of a straight line is, in general, a vertical
Darboux motion.

Some parts of this paper, mostly \autoref{sec:characterisation} and
the computations in the appendix, overlap with a previously published
conference paper \cite{rad15}. The investigation on the group of
admissible transformation in \autoref{sec:transformation-group}, the
characterisation of RP and PR dyads in
\autoref{sec:characterisation2}, and the relation of straight lines in
extended kinematic image space to vertical Darboux motions in
\autoref{sec:example} are new.

\section{Preliminaries}
\label{sec:preliminaries}

This article's scene is the projectivised dual quaternion model of
spatial kinematics. Here, we give a very brief introduction to this
model for the purpose of settling our notation. More details will be
introduced in the text as needed. For more thorough introductions to
dual quaternions and there relations to kinematics we refer to
Section~3 in Klawitter (2015) \cite{klawitter15} or Section~11 of
Selig (2005) \cite{selig05}.

The dual quaternions, denoted by $\D\H$, form an associative algebra
in $\R^8$ where multiplication of the basis elements $1$, $\qi$, $\qj$,
$\qk$, $\eps$, $\eps\qi$, $\eps\qj$, and $\eps\qk$ is defined by the
rules
\begin{gather*}
  \qi^2 = \qj^2 = \qk^2 = \qi\qj\qk = -1,\\
  \eps^2 = 0,\quad
  \qi\eps = \eps\qi,\quad
  \qj\eps = \eps\qj,\quad
  \qk\eps = \eps\qk.
\end{gather*}
An element $q \in \D\H$ may be written as $q = p + \eps d$ with
quaternions $p,d \in \H \coloneqq \langle 1, \qi, \qj, \qk \rangle$
(angled brackets denote linear span).  In this case the
\emph{quaternions} $p$ and $d$ are referred to as \emph{primal} and
\emph{dual} part of $q$, respectively. The conjugate dual quaternion
is $\cj{q} = \cj{p} + \eps\cj{d}$ and conjugation of quaternions is
done by multiplying the coefficients of $\qi$, $\qj$, and $\qk$ with
$-1$. It satisfies the rule $\cj{qr} = \cj{r}\,\cj{q}$ for any
$q,r \in \D\H$. The dual quaternion norm is defined as
$\Norm{q} = q\cj{q}$. We readily verify that it is a \emph{dual
  number,} that is, an element of
$\D \coloneqq \langle 1,\eps \rangle$.

We identify linearly dependent non-zero dual quaternions and thus
arrive at the projective space $P^7 = P(\R^8)$. Writing $[q]$ for the
point in $P^7$ that is represented by $q \in \D\H$, the \emph{Study
  quadric} is defined as
$\SQ \coloneqq \{[q] \in P^7\colon \Norm{q} \in \R \}$.  With
$q = p + \eps d$, the algebraic condition for $[q] \in \SQ$ is
$p\cj{d} + d\cj{p} = 0$.

Identifying $P^3$ with the projective subspace generated by
$\langle 1, \eps\qi, \eps\qj, \eps\qk \rangle$, a point
$[q] = [p + \eps d] \in \SQ$ with non-zero primal part acts on
$[x] \in P^3$ via
\begin{equation}
  \label{eq:1}
  [x] \mapsto [y] \coloneqq [(p + \eps d) x (\cj{p} - \eps \cj{d})].
\end{equation}
The map \eqref{eq:1} is the projective extension of a rigid body
displacement in $\R^3$. Composition of displacements corresponds to
dual quaternion multiplication.

The map that takes a point $[q] = [p + \eps d] \in \SQ \setminus \EG$
to the rigid body displacement \eqref{eq:1} is an isomorphism between
the factor group of dual quaternions of non-zero real norm modulo the
real multiplicative group and $\SE$. We refer to it as \emph{Study's
  kinematic map} or simply as \emph{kinematic map.}. It provides a
rich and solid algebraic and geometric environment for investigations
in kinematic.

\begin{remark}
  Provided $p \neq 0$, \eqref{eq:1} always describes a rigid body
  displacement, even if the Study condition is not fulfilled
  \cite{pfurner16}. In this case, the map that sends $p + \eps d$ to
  the rigid body displacement \eqref{eq:1} is no longer a group
  isomorphism but a homomorphism. We call it \emph{extended kinematic
    map.} In \autoref{sec:example} we will characterise the extended
  kinematic image of straight lines in~$P^7$.
\end{remark}

\section{Characterisation of the transformation group}
\label{sec:transformation-group}

The geometry in kinematic image space $P^7$ is invariant with respect
to coordinate changes in three-dimensional Euclidean space. More
precisely, it is invariant with respect to coordinate changes in both,
the frame of $[x]$ in \eqref{eq:1} (the \emph{moving frame}) and the
frame of $[y]$ in \eqref{eq:1} (the \emph{fixed frame}). The former
correspond to right-multiplications, the latter to
left-multiplications with dual quaternions of real norm and non-zero
primal part. Both transformations induce projective transformations in
$P^7$ and our aim in this section is a geometric characterisation of
the transformation group they generate. Necessary geometric conditions
on this group are already known but we are not aware of a formal proof
of sufficiency for a set of these conditions. This we will provide in
this section. It will refer to an important new geometric invariant,
the \emph{fiber projectivity,} whose usefulness we demonstrate in an
example. Later, it will re-appear in our characterisation of the
kinematic image of cylinder spaces.

For a quaternion $p = p_0 + p_1\qi + p_2\qj + p_3\qk$ we define
\begin{equation}
  \label{eq:2}
  \begin{gathered}
    \rM{p} \coloneqq
    \begin{bmatrix}
      p_0 & -p_1           & -p_2           & -p_3           \\
      p_1 & \phantom{-}p_0 & -p_3           & \phantom{-}p_2 \\
      p_2 & \phantom{-}p_3 & \phantom{-}p_0 & -p_1           \\
      p_3 & -p_2 & \phantom{-}p_1 & \phantom{-}p_0
    \end{bmatrix},\\
    \lM{p} \coloneqq
    \begin{bmatrix}
      p_0 & -p_1           & -p_2           & -p_3           \\
      p_1 & \phantom{-}p_0 & \phantom{-}p_3 & -p_2           \\
      p_2 & -p_3           & \phantom{-}p_0 & \phantom{-}p_1 \\
      p_3 & \phantom{-}p_2 & -p_1 & \phantom{-}p_0
    \end{bmatrix}.
  \end{gathered}
\end{equation}
With this notation, the projective maps $[x] \mapsto [px]$ and
$[x] \mapsto [xp]$ on $\DG = P^3$ may be written as
$[x_0,x_1,x_2,x_3]^\tp \mapsto \rM{p} \cdot [x_0,x_1,x_2,x_3]^\tp$ and
$[x_0,x_1,x_2,x_3]^\tp \mapsto \lM{p} \cdot [x_0,x_1,x_2,x_3]^\tp$,
respectively. Both leave invariant the quadric
$\EE\colon x\cj{x} = x_0^2 + x_1^2 + x_2^2 + x_3^2 = 0$ whence the
geometry in $\DG$, induced by coordinate changes in moving and fixed
frame, is that of elliptic three-space. In fact, the matrices in
\eqref{eq:2} describe the well-known Clifford right and left
translations that generate the transformation group of this space, see
Coxeter (1998), p.~140. \cite{coxeter98} Clifford right (left)
translations leave fixed every member of one family of (complex)
rulings of $\EE$ and we call those rulings \emph{right (left)
  rulings}, respectively.

Similarly, left-multiplication by a dual quaternion
$\ell_1 + \eps \ell_2$ and right-multiplication by a dual quaternion
$r_1 + \eps r_2$ can be effected by multiplication with matrices
$\LM{\ell_1+\eps \ell_2}$ and $\RM{r_1+\eps r_2}$, respectively. These
$8 \times 8$ matrices are conveniently described in terms of blocks of
dimension $4 \times 4$:
\begin{equation*}
  \LM{\ell_1+\eps \ell_2} =
  \begin{bmatrix}
    \lM{\ell_1}                                      & \zeromat \\
    \lM{\ell_2}                                      & \lM{\ell_1}
  \end{bmatrix},\quad
  \RM{r_1+\eps r_2} =
  \begin{bmatrix}
    \rM{r_1}                                      & \zeromat \\
    \rM{r_2}                                      & \rM{r_1}
  \end{bmatrix}.
\end{equation*}
Here, $\zeromat$ denotes the zero matrix of dimension $4 \times 4$. Two
matrices of above shape commute and their product
\begin{equation}
  \label{eq:3}
  \begin{aligned}
    T &= \LM{\ell_1+\eps \ell_2} \cdot \RM{r_1+\eps r_2} \\
      &= \begin{bmatrix}
           \lM{\ell_1}\cdot\rM{r_1}                         & \zeromat \\
           \lM{\ell_2}\cdot\rM{r_1} + \lM{\ell_1}\cdot\rM{r_2} & \lM{\ell_1}\cdot\rM{r_1}
         \end{bmatrix}
  \end{aligned}
\end{equation}
is the matrix of a general coordinate transformation which we want to
characterise geometrically. For that purpose, we introduce two further
invariant quadrics:

\begin{definition}
  The \emph{null cone $\NC$} is the quadric defined by the quadratic
  form $q = p + \eps d \mapsto p\cj{p}$.
\end{definition}

The points of the null cone $\NC$ are characterised by having purely
dual norm ($q\cj{q} \in \eps\R$). The null cone is a singular quadric
with three-dimensional vertex space
$\EG \coloneqq \{[\eps d]\colon d \in \H\}$ which we call the
\emph{exceptional generator}. It is contained in the Study quadric
$\SQ$ and the real points of $\NC$ are precisely those of $\EG$. By
$\EQ$ we denote the regular quadric in $\EG$ defined by the quadratic
form $\eps d \in \EG \mapsto d\cj{d}$. The quadrics $\EE \subset \DG$,
$\NC$, $\SQ$, and $\EQ \subset \EG$ are all invariant under
transformations of the shape \eqref{eq:3}.

\begin{theorem}
  \label{th:1}
  The transformation group described by matrices of shape \eqref{eq:3}
  where $\ell_1 + \eps \ell_2$ and $r_1 + \eps r_2$ satisfy the Study
  condition is characterised by the following properties: 1) It fixes
  any quadric in the pencil spanned by Study quadric $\SQ$ and null
  cone $\NC$ and 2) the restriction to $\EG$ preserves the two
  families of (complex) rulings of the quadric~$\EQ$.
\end{theorem}

\begin{proof}
  To begin with, it is elementary to verify that the conditions of the
  theorem are necessary for transformations given by \eqref{eq:3}. For
  the proof of sufficiency, we again employ block matrix
  notation. Assume that
  \begin{equation*}
    T =
    \begin{bmatrix}
      A & B \\ C & D
    \end{bmatrix}
  \end{equation*}
  is the matrix of a projective transformation
  $\tau\colon P^7 \to P^7$ with $4 \times 4$ blocks $A$, $B$, $C$, and
  $D$. The matrices of the pencil spanned by $\SQ$ and $\NC$ are of
  the shape
  \begin{equation}
    \label{eq:4}
    \begin{bmatrix}
      \nu \identity & \sigma \identity \\
      \sigma \identity & O
    \end{bmatrix}
  \end{equation}
  where $\identity$ and $O$ denote the $4 \times 4$ identity and zero
  matrix, respectively, and $\nu$ and $\sigma$ are real numbers. The
  only singular quadric in this pencil is $\NC$ and $\tau$ must fix
  its vertex space $\EG$ (the only three-dimensional space contained
  in $\NC$). Hence $B = O$ and $\tau$ transforms the matrix of $\NC$
  to
  \begin{equation*}
    \begin{bmatrix}
      A^\tp & C^\tp \\ O & D^\tp
    \end{bmatrix} \cdot
    \begin{bmatrix}
      \identity & O \\ O & O
    \end{bmatrix} \cdot
    \begin{bmatrix}
      A & O \\ C & D
    \end{bmatrix} =
    \begin{bmatrix}
      A^\tp A & O \\ O & O
    \end{bmatrix}
  \end{equation*}
  whence $A$ is the scalar multiple of an orthogonal matrix. The
  transformed matrix of $\SQ$ is
  \begin{equation*}
    \begin{bmatrix}
      A^\tp & C^\tp \\ O & D^\tp
    \end{bmatrix} \cdot
    \begin{bmatrix}
      O & \identity \\ \identity & O
    \end{bmatrix} \cdot
    \begin{bmatrix}
      A & O \\ C & D
    \end{bmatrix} =
    \begin{bmatrix}
      C^\tp A + A^\tp C & A^\tp D \\
      D^\tp A & O
    \end{bmatrix}.
  \end{equation*}
  From this we infer that $D = \alpha A$ for some $\alpha \in \R$ and
  $C^\tp A + A^\tp C = O$. This latter matrix equation has the general
  solution $C = A^{-\tp}S$ with an arbitrary skew-symmetric matrix $S$
  of dimension $4 \times 4$. Finally, we study the action on a third
  quadric in the pencil:
  \begin{multline*}
    \begin{bmatrix}
      A^\tp & C^\tp \\ O & \alpha A^\tp
    \end{bmatrix} \cdot
    \begin{bmatrix}
      \identity & \identity \\ \identity & O
    \end{bmatrix} \cdot
    \begin{bmatrix}
      A & O \\ C & \alpha A
    \end{bmatrix} \\=
    \begin{bmatrix}
      A^\tp A + C^\tp A + A^\tp C & \alpha A^\tp A \\
      \alpha A^\tp A & O
    \end{bmatrix} =
    \begin{bmatrix}
      A^\tp A & \alpha A^\tp A \\
      \alpha A^\tp A & O
    \end{bmatrix}.
  \end{multline*}
  This implies $\alpha = 1$ whence
  \begin{equation}
    \label{eq:5}
    T =
    \begin{bmatrix}
      A & O \\
      A^{-\tp} S & A
    \end{bmatrix}.
  \end{equation}
  The restriction of $\tau$ to $\EG$ transforms the matrix
  $[\identity]$ of $\EQ \subset \EG$ to
  $[A^\tp \cdot \identity \cdot A] = [\identity]$ and hence fixes
  $\EQ$. The condition on the rulings implies $\det A > 0$. By
  well-known results of three-dimensional elliptic geometry, $A$ is
  uniquely expressible as product of a Clifford left translation and a
  Clifford right translation (Page~140 of Coxeter 1998
  \cite{coxeter98}). In our notation, this means that there exist
  quaternions $\ell_1$, $r_1$ such that
  $A = \lM{\ell_1} \cdot \rM{r_1}$. Hence, the upper left and lower
  right corners of \eqref{eq:5} and \eqref{eq:3} match. For given $A$,
  the set of possible matrices $C = A^{-\tp}S$ is a real vector space
  of dimension six. The same is true for the $4 \times 4$ sub-matrices
  in the lower left corner of the matrix \eqref{eq:3} (recall that
  $\ell_1\cj{\ell_2} + \ell_2\cj{\ell_1} = r_1\cj{r_2} + r_2\cj{r_1} =
  0$). Hence, there exist unique quaternions $\ell_2$, $r_2$ such that
  $C = \lM{\ell_2} \cdot \rM{r_1} + \lM{\ell_1} \cdot \rM{r_2}$ and
  the theorem is proved.
\end{proof}

Examining the proof, it is easy to see that the system of invariants
in \autoref{th:1} is minimal.

Invariance of left and right rulings, respectively, of $\EQ$ is an
important property that is, for example, responsible for different
kinematic properties of motion and inverse motion. For the planar
case, this is mentioned in Bottema and Roth (1990), Chapter~11,~\S14
\cite{bottema90}. The map $\chi\colon P^7 \to P^7$,
$[q] \mapsto [\cj{q}]$ (quaternion conjugation) is a projective
transformation of $P^7$ leaving invariant the quadrics of the pencil
spanned by $\SQ$ and $\NC$ (lets call it the \emph{absolute pencil})
but it interchanges the left and right rulings of $\EQ$. Given a
motion $\gamma$ (a curve in $\SQ$), the inverse motion is
$\cj{\gamma} \coloneqq \chi(\gamma)$. It is of the same algebraic
degree and has the same relative position to absolute pencil (in
projective sense) but not necessarily the same position to the left
and right rulings of~$\EQ$. In order to demonstrate this at hand of an
example, we introduce a further invariant concept.

\begin{definition}
  \label{def:1}
  The projective map 
  \begin{equation}
    \label{eq:6}
    \fiberproj\colon P^7 \to \EG,\quad [x'+\eps x''] \mapsto [\eps x']
  \end{equation}
  that assigns to each point in $P^7$ the projection on its primal
  part times $\varepsilon$ is called the \emph{fiber projectivity.}
\end{definition}

The name ``fiber projectivity'' is motivated by a recently introduced
non-injective extensions of Study's kinematic map \cite{pfurner16}
which reads as in \eqref{eq:1} but drops the Study condition
$p\cj{d} + d\cj{p} = 0$. Its fibers are obtained by connecting a point
$[x]$ with $\fiberproj([x])$. The fiber projectivity is invariant with
respect to transformations $\tau$ given by \eqref{eq:3} in the sense
that $\tau \circ \varphi([x]) = \varphi \circ \tau([x])$ for every
point $[x] \in P^7$. The restriction of the fiber projectivity to
$\DG$ is a bijection in which $\EE$ and $\EQ$ correspond so that we
may speak of \emph{right (left) rulings} of $\EQ$ as well.

\begin{example}
  \label{ex:1}
  The \emph{Darboux motion} is the only spatial motion with planar
  trajectories, see Bottema and Roth, 1990, Equation~(3.4)
  \cite{bottema90,li15}. A parametric equation in dual quaternions
  reads
  \begin{equation*}
    C(t) = (c\eps+\qk)t^3 + (1+\eps(b - a\qi - c\qk))t^2 +(\qk - \eps(a\qj + b\qk))t + 1
  \end{equation*}
  (see Li et al., 2015 \cite{li15}). The motion parameter is $t$ while
  $a$, $b$, and $c$ are constant real numbers. This rational cubic
  curve intersects $\EG$ in the two points
  \begin{equation*}
    [d_{1,2}] \coloneqq [\eps(\pm c\ci+b-a\qi \pm a\ci\qj+(\pm b\ci-c)\qk)]
  \end{equation*}
  which even lie on $\EQ$. (Note that ``$\ci$'' denotes a complex
  number which must not be confused with the quaternion unit
  ``$\qi$''.) The fiber projection of $[C(t)]$ is the point
  \begin{equation*}
    \fiberproj([C(t)]) = [(1+t^2)\eps(\qk t + 1)] = [\eps(\qk t + 1)].
  \end{equation*}
  For varying $t$, these points vary on a straight line which
  intersects $\EQ$ in the two points
  $[f_{1,2}] \coloneqq [\eps(1 \pm \ci\qk)]$. It is now easy to verify
  that the lines $[d_1] \vee [f_1]$ and $[d_2] \vee [f_2]$ are rulings
  of $\EQ$. Moreover, $[d_1] = [pf_1]$ and $[d_2] = [pf_2]$ where
  $p = -b + a\qi + c\qk$. This means that a Darboux motion intersects
  $\EQ$ in two points $[d_1]$, $[d_2]$ and its fiber projection
  intersects $\EQ$ in two points $[f_1]$, $[f_2]$ such that the lines
  $[d_1] \vee [f_1]$ and $[d_2] \vee [f_2]$ are \emph{left rulings} of
  $\EQ$. The inverse motion
  \begin{equation*}
    \cj{C}(t) = (c\eps-\qk)t^3 + (1+\eps(b + a\qi + c\qk))t^2 -(\qk - \eps(a\qj + b\qk))t + 1
  \end{equation*}
  (\emph{Mannheim motion}) has similar properties as curve in $P^7$
  but different kinematic properties. Most notably, the trajectories
  of $C(t)$ are rational of degree \emph{two} while those of
  $\cj{C}(t)$ are rational of degree \emph{four}. The deeper reasons
  for this is the non-invariance of left and right rulings with
  respect to the quaternion conjugation map $\chi$. Hence, the
  intersection points of $\cj{C}$ and $\fiberproj(\cj{C})$ with $\EQ$
  span \emph{right} rulings. The case of a \emph{vertical Darboux
    motion} ($a = 0$) is special. Here we have $[d_1] = [f_1]$ and
  $[d_2] = [f_2]$ and the distinction between left and right rulings
  vanishes. Indeed, the inverse motion is again a vertical Darboux
  motion.
\end{example}

\section{A characterisation of 2R spaces}
\label{sec:characterisation}

Given two non co-planar lines $\ell_1$, $\ell_2$ in Euclidean
three-space, we consider the set of all displacements obtained as
composition of a rotation around $\ell_2$, followed by a rotation
about $\ell_1$. Its kinematic image is known to lie in a three-space
(Selig 2005, Section~11.4) \cite{selig05} which we call a \emph{2R
  space}. It always contains the displacement corresponding to zero
rotation angle around both axes. It is no loss of generality to view
it as identity displacement $[1]$ which we will often do in proofs
(and consistently did in a previous publication\cite{rad15}). However,
we avoid this in our theorems because it is not invariant with respect
to the transformations characterised in \autoref{th:1}. Fixing one
rotation angle and varying the other yields a straight line in $\SQ$.
The thus obtained two families of lines form the rulings of a quadric
surface in $\SQ$. In order to characterise 2R spaces among all
three-dimensional subspaces of $P^7$ with these properties, we
introduce the following notions:

\begin{definition}
  A \emph{spatial quadrilateral} is a set of four different lines
  $\{\ell_0,\ell_1,\ell_2,\ell_3\}$ in projective space such that the
  intersections $\ell_0 \cap \ell_1$, $\ell_1 \cap \ell_2$,
  $\ell_2 \cap \ell_3$, and $\ell_3 \cap \ell_0$ are not empty. The
  lines $\ell_0$, $\ell_1$, $\ell_2$, and $\ell_3$ are called the
  quadrilateral's \emph{edges}. A \emph{null line} is a straight line
  contained in the Study quadric $\SQ$ and the null cone $\NC$. A
  \emph{null quadrilateral} is a spatial quadrilateral whose elements
  are \emph{null lines.}
\end{definition}

\begin{theorem}
  \label{th:2}
  A three-space $U \subset P^7$ is a 2R space if and only if it
  \begin{itemize}
  \item intersects the Study quadric in a regular ruled quadric~$\QQ$,
  \item does not intersect the exceptional three-space $\EG$, and
  \item contains a null quadrilateral.
  \end{itemize}
\end{theorem}

The first and second item in \autoref{th:2} exclude exceptional cases
with co-planar, complex, or ``infinite'' revolute axes. The latter
correspond to prismatic joints and will be treated later. The crucial
point is existence of a null quadrilateral. We split the proof of
\autoref{th:2} into a series of lemmas. It will be finished by the end
of this section.

\begin{lemma}
  \label{lem:1}
  The straight line $[x] \vee [y]$ is contained in $\SQ \cap \NC$ if
  and only if $x\cj{x} = y\cj{y} = x\cj{y} + y\cj{x} = 0$.
\end{lemma}
We omit the straightforward computational proof of \autoref{lem:1}.
Note that the left-hand sides of each of the three conditions in this
lemma are dual numbers. Hence, they give \emph{six} independent linear
equations for the real coefficients of $x$ and~$y$.

\begin{lemma}
  \label{lem:2}
  The conditions of \autoref{th:2} are necessary for 2R spaces.
\end{lemma}

\begin{proof}
  The quadric $\QQ$ contains a regular point $[u]$ and, without loss
  of generality, we may assume $[u] = [1]$. Then, the constraint
  variety of a 2R chain can be parameterised as
  \begin{equation}
    \label{eq:7}
    R(t_1,t_2) = (t_1 - h_1)(t_2-h_2)
  \end{equation}
  with two dual quaternions $h_1, h_2$ that satisfy
  $h_1\cj{h_1} = h_2\cj{h_2} = 1$, $h_1+\cj{h_1} = h_2+\cj{h_2} = 0$
  and $h_1h_2 \neq h_2h_1$. The first condition ensures that the dual
  quaternions are suitably normalised and $[h_1], [h_2] \in \SQ$. The
  second condition means that $h_1$ and $h_2$ describe half-turns
  (rotations through an angle of $\pi$). The third condition is true
  because the axes of these half turns are not co-planar. (Two half
  turns commute if and only if their axes are identical or orthogonal
  and co-planar.) Equation~\eqref{eq:7} describes a composition of two
  rotations for all values of $t_1$, $t_2$ in $\R \cup \{\infty\}$
  with $\infty$ corresponding to zero rotation angle. Even if their
  kinematic meaning is unclear, we also allow complex parameter
  values. Expanding \eqref{eq:7} yields
  $R(t_1,t_2) = t_1t_2 - t_1h_2 - t_2h_1 + h_1h_2$. We see that the
  kinematic image of the 2R dyad lies in the three-space spanned by
  $[1]$, $[h_1]$, $[h_2]$, $[h_1h_2]$. In a suitable projective
  coordinate frame with these points as base points, we use projective
  coordinates $[x_0,x_1,x_2,x_3]$. Then, the surface parameterisation
  \eqref{eq:7} reads $x_0 = t_1t_2$, $x_1 = -t_1$, $x_2 = -t_2$,
  $x_3 = 1$. It describes the quadric with equation
  $x_0x_3 - x_1x_2 = 0$ which is indeed regular and ruled.

  The intersection of $U$ with the exceptional three-space $\EG$ is
  non empty if and only if the primal part of $R(t_1,t_2)$ vanishes
  for certain parameter values $t_1$, $t_2$. This can only happen if
  the primal parts of $h_1$ and $h_2$ are linearly dependent over $\R$
  but then the revolute axes are parallel and $U$, contrary to our
  assumption, is contained in $\SQ$. The other possibility for
  $U \subset \SQ$, intersecting revolute axes, has been excluded as
  well. Clearly, $U$ is not contained in the null cone $\NC$ either.
  We claim that the intersection of $U$ and $\NC$ consists of the four
  lines given by $t_1 = \pm\ci$, $t_2 = \pm\ci$. Indeed, they are null
  lines. Set, for example, $z(t_2) \coloneqq (\ci-h_1)(t_2-h_2)$,
  $x \coloneqq z(0)$ and $y \coloneqq z(1)$. In view of
  \autoref{lem:1}, we have to verify $x\cj{x} = y\cj{y} = 0$. But this
  follows from
  \begin{equation*}
    \begin{aligned}
      z(t_2)\cj{z(t_2)} &= (\ci-h_1)(t_2-h_2)(t_2-\cj{h_2})(\ci-\cj{h_1}) \\
              &= (t_2-h_2)(t_2-\cj{h_2})(\ci-h_1)(\ci-\cj{h_1}) \\
              &= (t_2-h_2)(t_2-\cj{h_2})(-\ci^2-\ci(h_1+\cj{h_1})+h_1\cj{h_1}) \\
              &= (t_2-h_2)(t_2-\cj{h_2})(-1-0+1) = 0.
    \end{aligned}
  \end{equation*}
  Here, we used the fact that $(t_2-h_2)(t_2-\cj{h_2})$ is a real
  number and thus commutes with all other factors. The cases
  $t_1=-\ci$, $t_2=\pm\ci$ are similar so that we have verified all
  conditions of \autoref{th:2}.
\end{proof}

The proof of sufficiency is more involved. We need two additional
lemmas from projective geometry which are formulated and proved in
\autoref{sec:appendix}.

\begin{lemma}
  \label{lem:3}
  A three-space $U$ that satisfies all conditions of \autoref{th:2} is
  a 2R space.
\end{lemma}

\begin{proof}
  Denote the vertices of the null quadrilateral by $[u_1]$, $[v_1]$,
  $[u_2]$, $[v_2]$ such that consecutive points span the
  quadrilateral's edges. Again, it is no loss of generality to assume
  $[1] \in \QQ$. We denote the tangent hyperplane of $\SQ$ at $[1]$ by
  $\TO$. A point $[p'+\eps p'']$ is contained in $\TO$ if and only if
  $p''+\cj{p''} = 0$. Since $[1]$ does not lie on the null
  quadrilateral, there exist (possibly not unique) points
  \begin{gather*}
    [m_1] \in [u_1] \vee [v_1],\quad
    [n_1] \in [v_1] \vee [u_2],\\
    [m_2] \in [u_2] \vee [v_2],\quad
    [n_2] \in [v_2] \vee [u_1]
  \end{gather*}
  in the intersection of $U$ and $\TO$ such that
  \begin{itemize}
  \item $[m_1]$, $[m_2]$ and $[n_1]$, $[n_2]$ are pairs of complex
    conjugate points (whence their respective joins are real lines)
    and
  \item the quadrilateral with these points as vertices is planar and
    non-degenerate (\autoref{fig:null-lines}).
  \end{itemize}
  The real points of $[m_1] \vee [m_2]$ and $[n_1] \vee [n_2]$
  correspond to rotations about fixed axes. Pick real rotation
  quaternions $h_1$, $h_2 \neq 1$ such that
  $[h_1] \in [m_1] \vee [m_2]$ and $[h_2] \in [n_1] \vee [n_2]$. The
  axes of these rotations are the only candidates for our 2R dyad.
  Hence, we have to show that either $[h_1h_2] \in U$ or
  $[h_2h_1] \in U$. It is easy to see that this is the case if and
  only if either $[m_1n_1] \in U$ or $[n_1m_1] \in U$. In fact, we
  will even show that one of these products equals~$v_1$.

  \begin{figure}
    \centering
    \includegraphics{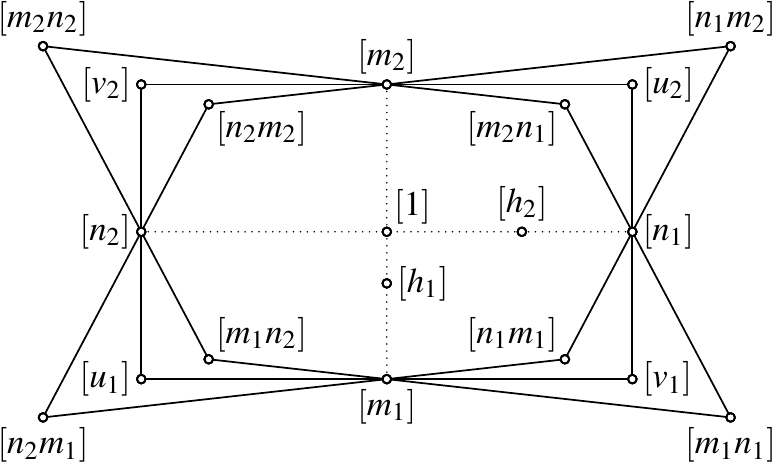}
    \caption{Null lines and null quadrilaterals in the proof of
      \autoref{lem:3}.}
    \label{fig:null-lines}
  \end{figure}

  We claim that both $M_1 \coloneqq [m_1] \vee [m_1n_1]$ and
  $M_2 \coloneqq [m_1] \vee [n_1m_1]$ are null lines. In order to show
  this, we have to verify the conditions of \autoref{lem:1}:
  \begin{gather*}
    (m_1n_1)\cj{(m_1n_1)} = m_1(\underbrace{n_1\cj{n_1}}_{=0})\cj{m_1} = 0,\\
    (n_1m_1)\cj{(n_1m_1)} = n_1(\underbrace{m_1\cj{m_1}}_{=0})\cj{n_1} = 0,\\
    m_1\cj{(m_1n_1)} + (m_1n_1)\cj{m_1} =
    m_1(\underbrace{n_1 + \cj{n_1}}_{\in \C})\cj{m_1} = (n_1+\cj{n_1})(\underbrace{m_1\cj{m_1}}_{=0}) = 0,\\
    m_1\cj{(n_1m_1)} + (n_1m_1)\cj{m_1} =
    (\underbrace{m_1\cj{m_1}}_{=0})\cj{n_1} + n_1(\underbrace{m_1\cj{m_1}}_{=0}) = 0.
  \end{gather*}
  Similarly, we see that also $N_1 \coloneqq [n_1] \vee [m_1n_1]$ and
  $N_2 \coloneqq [n_1] \vee [n_1m_1]$ are null lines. Thus, we are in
  the situation depicted in \autoref{fig:null-lines} where we have
  three null quadrilaterals with respective vertices
  \begin{gather*}
    [u_1], [v_1], [u_2], [v_2];\quad
    [m_1n_1], [m_2n_1], [m_2n_2], [m_1n_2];\\
    [n_1m_1], [n_1m_2], [n_2m_2], [n_2m_1].
  \end{gather*}
  The second and third quadrilateral are different because $m_1$ and
  $n_1$ do not commute (otherwise they would lie on the same line
  through $[1]$ which contradicts the regularity of the quadric
  $\QQ \coloneqq U \cap \SQ$). Our proof will be finished as soon as we
  have shown that the first and the second or the first and the third
  quadrilateral are equal. For this, it is sufficient to show that
  $[v_1] = [m_1n_1]$ or $[v_1] = [n_1m_1]$.

  At first, we argue that the primal part of $[v_1]$ equals the primal
  part of $[m_1n_1]$ or of $[n_1m_1]$. Because $U$ does not intersect
  $\EG$, the projection on the primal part is a regular projectivity
  $U \to [\H]$ with centre $\EG$. We denote projected objects by a
  prime, that is, we write $u'_1$, $v'_1$, $m'_1$ etc., for the primal
  parts of $u_1$, $v_1$, $m_1$ etc. The quadric $\QQ$ is regular and
  ruled and so is its primal projection $\QQ'$. Hence, the point
  $[m'_1] \in \QQ'$ is incident with precisely two lines contained in
  $\QQ'$, say $M'_1$ and $M'_2$. But $[m'_1] \vee [v'_1]$ is contained
  in $\QQ'$. Hence $[v'_1] \in M'_1$ or $[v'_1] \in M'_2$.  Similarly,
  $[v'_1]$ is also incident with one of the two lines contained in
  $\QQ'$ and incident with $[n'_1]$. Thus, $[v'_1] = [m'_1n'_1]$ or
  $[v'_1] = [n'_1m'_1]$.

  Now we have to lift this result to the dual part and show that
  $[v_1] = [m_1n_1]$ or $[v_1] = [n_1m_1]$. The alternative being
  similar, we assume $[v'_1] = [m'_1n'_1]$. This means that we have
  two null quadrilaterals, one with vertices $[u_1]$, $[v_1]$,
  $[u_2]$, $[v_2]$ and one with vertices $[m_1n_1]$, $[m_2n_1]$,
  $[m_2n_2]$, $[m_1n_2]$ such that their primal projections are equal
  and corresponding sides intersect in the vertices $[m_1]$, $[n_1]$,
  $[m_2]$, $[n_2]$ of a planar quadrilateral. Now we wish to apply
  \autoref{lem:7} in \autoref{sec:appendix} with $E = \EG$, $F = \DG$,
  and $\QQ = \SQ$ in order to conclude $[v_1] = [m_1n_1]$. This is
  admissible if and only if the plane
  $L \coloneqq [1] \vee [h_1] \vee [h_2] = [1] \vee [m_1] \vee [n_1]$
  does not intersect the four-spaces $[u_1] \vee \EG$,
  $[v_1] \vee \EG$, $[u_2] \vee \EG$, and $[v_2] \vee \EG$. We
  consider the four-space $[v_1] \vee \EG = [m_1n_1] \vee \EG$. It
  does not intersect $L$ if and only if the linear combination
  \begin{equation}
    \label{eq:8}
    \alpha \cdot 1 + \beta m_1 + \gamma n_1 + \delta m'_1n'_1 +
    \sum_{\ell=0}^3 \varphi_\ell w_\ell = 0
  \end{equation}
  with some basis $(w_0,w_1,w_2,w_3)$ of $\eps\H$ is trivial. By
  considering only primal parts, we see that \eqref{eq:8} implies
  \begin{equation*}
    \alpha \cdot 1 + \beta m'_1 + \gamma n'_1 + \delta m'_1n'_1 = 0.
  \end{equation*}
  But the points $[1]$, $[m'_1]$, $[n'_1]$, $[m'_1n'_1]$ are vertices
  of a quadrilateral contained in a regular ruled quadric and
  therefore not co-planar. Thus $\alpha = \beta = \gamma = \delta = 0$
  and $\varphi_0 = \varphi_1 = \varphi_2 = \varphi_3 = 0$ follows. The
  arguments for the remaining four-spaces are similar.
\end{proof}

\section{A characterisation of RP, PR, and cylinder spaces}
\label{sec:characterisation2}

Now we proceed with a geometric characterisation of the images of RP
and PR dyads. It is tempting to view them as limiting case of RR dyads
where one of the revolute axes becomes ``infinite''. However, this is
not entirely justified because of the possibility of commuting joints
-- a phenomenon which is trivial for RR dyads but relevant for RP or
PR dyads. It happens precisely if the direction of the revolute axis
and translation direction are linearly dependent. These dyads are
special enough to deserve a name of their own. Since their
transformations are usually modelled by a ``cylindrical joint'', we
define:

\begin{definition}
  The projective span of the kinematic image of an RP or PR dyad is
  called an \emph{RP} or \emph{PR space} respectively. If rotation and
  translation commute, we speak of a \emph{cylinder space} or \emph{C
    space.}
\end{definition}

The following example demonstrates that RP and PR spaces cannot be
characterised by a limiting configuration of \autoref{th:3}.

\begin{example}
  \label{ex:2}
  Consider the three-space $U$ spanned by $[1]$, $[m_1]$, $[n_1]$ and
  $[s_1]$ where
  \begin{equation*}
    m_1 = \eps\qi,\quad
    n_1 = \ci + \qi,\quad
    s_1 = \eps(\ci + \qi + \qj + \ci\qk).
  \end{equation*}
  It is straightforward to verify that the line $[m_1] \vee [s_1]$ is
  a null line contained in $\EG$ and the line $[n_1] \vee [s_1]$ and
  its conjugate complex line are null lines not contained in $\EG$.
  The three-space $U$ contains all rotations about $\qi$ and all
  translations in direction of $\qi$. The corresponding RP space is
  generated by the regular quadric with parametric equation
  \begin{equation*}
    R(u,v) = (u - \qi)(v - \eps\qi)
  \end{equation*}
  where $u$ and $v$ both range in $\R \cup \{\infty\}$. This quadric
  is not contained in $U$ which is hence neither an RP nor a PR spac.
\end{example}

It is crucial to above example that $m_1$ and $n_1$ commute. This can
also be seen in our characterisation of RP and PR spaces where C
spaces require a special treatment.

\begin{theorem}
  \label{th:3}
  A three-space $U \subset P^7$ is an RP but not a C space if and only
  if it
  \begin{itemize}
  \item intersects the Study quadric in a regular ruled quadric~$\QQ$,
  \item intersects the exceptional three-space $\EG$ in a straight
    line~$e_1$, and
  \item contains a pair of conjugate complex null lines $\ell_1$,
    $\ell_2$ that intersect $e_1$ in points $[s_1]$, $[s_2]$
    respectively, such that $\fiberproj(\ell_1) \vee [s_1]$ and
    $\fiberproj(\ell_2) \vee [s_2]$ are right rulings of~$\EQ$.
  \end{itemize}
  RP spaces that are not C spaces are characterised by the same
  conditions but with ``left rulings'' instead of ``right rulings''.
\end{theorem}

\begin{figure}
  \centering
  \includegraphics{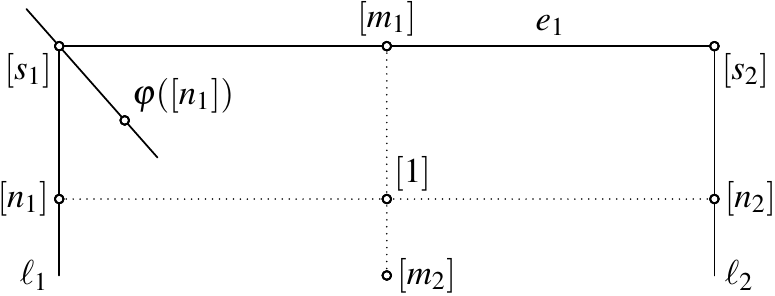}
  \caption{Lines in the proof of \autoref{th:4} (\autoref{lem:4} and
    \autoref{lem:5}).}
  \label{fig:null-lines2}
\end{figure}

Necessity of the first and second condition of \autoref{th:3} have
been proved in the master thesis by Stigger \cite{stigger15} by means
of Gröbner basis computations. Here, we give an independent proof.

\begin{lemma}
  \label{lem:4}
  The conditions of \autoref{th:3} are necessary for an RP or PR
  space that is not a C space.
\end{lemma}

\begin{proof}
  Assuming without loss of generality $[1] \in U$, the constraint
  variety of an RP chain can be parameterised as
  \begin{equation}
    \label{eq:9}
    R(u,v) = (u - h)(v - \eps p),
    \quad
    u \in \R \cup \{\infty\},\
    v \in \R \cup \{\infty\}
  \end{equation}
  with a non-zero quaternion $p$ with $p + \cj{p} = 0$ and a dual
  quaternion $h$ that satisfies $h\cj{h} = 1$, $h + \cj{h} = 0$. Since
  it is not the constraint variety of a PR chain, the primal part $h'$
  of $h$ and $p$ do not commute, that is $[h'] \neq [p]$.

  By the same arguments as in \autoref{lem:2}, this constraint variety
  is a regular quadric: The parameters $u$ and $v$ range in
  $\R \cup \{\infty\}$ (with $\infty$ corresponding to zero rotation
  angle or translation distance) but we will also admit complex values
  for them. Expanding \eqref{eq:9} yields
  $R(u,v) = uv - vh - u\eps p + \eps hp$. We see that the kinematic
  image of this RP dyad lies in the three-space spanned by $[1]$,
  $[h]$, $[\eps p]$ and $[\eps hp] = [\eps h'p]$ and is a regular
  quadric.

  Clearly, this quadric it contains the line
  $e_1 \coloneqq [\eps p] \vee [\eps hp] \subset \EG$. Moreover, the
  lines $\ell_1$, $\ell_2$ parameterised by $R(\pm\ci, v)$ are
  conjugate complex null lines. The intersection points $[n_1]$, $[n_2]$
  of these lines with $\TO$ correspond to $v = \infty$, their
  intersection points $[s_1]$, $[s_2]$ with $\EG$ correspond to $v = 0$:
  \begin{equation*}
    n_1 = \ci - h,\
    n_2 = -\ci - h,\
    s_1 = -\eps(\ci - h)p,\
    s_2 = -\eps(-\ci - h)p
  \end{equation*}
  (\autoref{fig:null-lines2}). In view of \autoref{lem:1} the null
  line property follows from
  \begin{equation*}
    \begin{aligned}
    s_1\cj{n_1} + n_1\cj{s_1}
    &= -\eps(\ci-h)p(\ci-\cj{h}) + (\ci-h)\cj{p}(-\eps(\ci-\cj{h})) \\
    &= -(\ci-h)\eps(\underbrace{\cj{p}+p}_{=0})(\ci-\cj{h}) = 0
    \end{aligned}
  \end{equation*}
  and similar with $\ci$ replaced by $-\ci$. Moreover
  $\fiberproj([\ell_1]) = \fiberproj([n_1]) = [\eps(\ci - h')]$ and
  $[s_1] = [\eps(\ci-h')p]$. Because of $[h'] \neq [p]$, these points
  do not coincide and, indeed, span a right ruling of~$\EQ$.

  The statements for PR spaces follow from \autoref{th:1} by
  application of the conjugation map $\chi$. It interchanges RP spaces
  with PR spaces and right rulings with left rulings but fixes the
  Study quadric and the null cone.
\end{proof}

The proof of sufficiency of the conditions in \autoref{th:3} is, in
large parts, a copy of \autoref{lem:3}. However, at one point the
third condition comes into play.

\begin{lemma}
  \label{lem:5}
  A three-space $U$ that satisfies all conditions of \autoref{th:3} is
  a PR or RP space but not a C space.
\end{lemma}

\begin{proof}
  Once more, we assume $[1] \in U$. The intersection of $U$ and $\SQ$
  is a regular quadric $\QQ$ that contains the straight line
  $e_1 = U \cap \EG$ and the two conjugate complex null lines
  $\ell_1$, $\ell_2$. There exist points $[m_1] \in e_1$ and
  $[n_1] \in \ell_1$ such that $[1] \vee [m_1]$ and $[1] \vee [n_1]$
  are rulings of $\QQ$. The points of the line $[1] \vee [m_1]$
  correspond to translations with fixed direction and the points of
  $[1] \vee [n_1]$ to rotations around a fixed axis. There exists an
  intersection point $[s_1] \coloneqq e_1 \cap \ell_1$ and we have to
  show that either $[s_1] = [m_1n_1]$ or $[s_1] = [n_1m_1]$. In
  general, this can be done just as in our proof of \autoref{lem:3}:
  Denote by $m_2$ the complex conjugate to $m_1$ and by $n_2$ the
  complex conjugate of $n_1$. Then
  $[1] = ([m_1] \vee [m_2]) \cap ([n_1] \vee [n_2])$, the
  quadrilaterals with respective vertices $[m_1n_1]$, $[m_2n_1]$,
  $[m_2n_2]$, $[m_1n_2]$ and $[n_1m_1]$, $[n_1m_2]$, $[n_2m_2]$,
  $[n_2m_1]$ are contained in $\QQ$, and the points $[m_1]$, $[n_1]$,
  $[m_2]$, and $[n_2]$ form a planar quadrilateral. Because the
  intersection of $\EG$ and $U$ is not empty, we have to use a
  projection with three-dimensional centre $Z$ that does not intersect
  $U$ (and hence differs from $\EG$) and is such that the plane
  $[1] \vee [m_1] \vee [n_1]$ is complementary to the four-spaces
  \begin{equation*}
    [u_1] \vee Z,\quad
    [v_1] \vee Z,\quad
    [u_2] \vee Z,\quad
    [v_2] \vee Z
  \end{equation*}
  when appealing to \autoref{lem:7}. Such a choice is possible and
  then the differences to the proof of \autoref{lem:3} are irrelevant
  \emph{unless $m_1$ and $n_1$ commute.} In this case our argument
  fails because the two spatial quadrilaterals coincide and we may no
  longer conclude that the projection of $[s_1]$ equals the projection
  of $[m_1n_1]$ or of $[n_1m_1]$. Fortunately, this case is already
  special enough to allow a straightforward computational treatment.

  The dual quaternions $m_1$ and $n_1$ commute ($m_1n_1 = n_1m_1$) if
  the translation direction of $m_1$ and the direction of the rotation
  axis of $n_1$ are linearly dependent. Without loss of generality, we
  may set $m_1 \coloneqq \eps \qi$ and $n_1 \coloneqq \ci + \qi$ so
  that $m_1n_1 = n_1m_1 = \eps(\ci\qi - 1)$. Now $U$ equals the span
  of $[1]$, $[m_1]$, $[n_1]$, and $[s_1] = [m_1n_1] = [n_1m_1]$ but
  $\fiberproj(\ell_1) = \fiberproj([n_1]) = [\eps(\ci+\qi)] =
  [m_1n_1]$ -- a contraction to the assumption that $[s_1]$ and
  $\fiberproj([n_1])$ span a straight line.
\end{proof}

Above proof already contains the basic idea for the still missing
characterisation of C spaces.

\begin{theorem}
  \label{th:4}
  A three-space $U \subset P^7$ is a C space if and only if it
  \begin{itemize}
  \item intersects the Study quadric in a regular ruled quadric $\QQ$,
  \item contains a pair of conjugate complex null lines $\ell_1$,
    $\ell_2$ and a real null line $e_1$, and
  \item satisfies $\fiberproj(\QQ) = e_1$.
  \end{itemize}
\end{theorem}

\begin{proof}
  Lets assume that $U$ is a C space. Without loss of generality, we
  may assume that it is spanned by $[1]$, $[m_1] = [\eps\qi]$,
  $[n_1] = [\ci + \qi]$ and $[s_1] = [m_1n_1] = \fiberproj([n_1])$.
  Clearly, it contains the real line $e_1 = [m_1] \vee [s_1]$, the
  null line $\ell_1 = [s_1] \vee [n_1]$, and its complex conjugate
  $\ell_2$. A parametric equation for the intersection quadric $\QQ$
  of $U$ and $\SQ$ is
  \begin{equation*}
    R(u,v) = (u - n_1)(v - m_1),
    \quad
    u \in \R \cup \{\infty\},\
    v \in \R \cup \{\infty\}.
  \end{equation*}
  It is a regular ruled quadric and its fiber projection equals
  $\fiberproj(\QQ)$ is parameterized by $\eps(u-s_1)$. This is
  indeed the straight line $e_1$.

  Conversely, assume that the condition of the theorem hold and,
  without loss of generality, $[1] \in U$. As in the proof of
  \autoref{lem:5}, we can find points $[m_1]$, $[n_1]$, $[s_1] \in U$
  such that $e_1 = [m_1] \vee [s_1]$, $\ell_1 = [n_1] \vee [s_1]$, and
  the lines $[1] \vee [m_1]$ and $[1] \vee [n_1]$ are rulings of
  $\QQ$. After a suitable choice of coordinates, we may assume $n_1 =
  \ci + \qi$ whence
  \begin{equation*}
    e_1 = \fiberproj([1]) \vee \fiberproj([n_1]) =
    [\eps] \vee [\eps(\ci + \qi)] = [\eps] \vee [\eps\qi].
  \end{equation*}
  But then $[m_1] = [\eps\qi]$ because $[m_1] \in e_1$ and
  $[1] \vee [m_1]$ is contained in $\SQ$. In particular,
  $[m_1n_1] = [n_1m_1] = [\eps(\ci\qi - 1)] = [\eps(\ci+\qi)]$.
  Moreover, there exist scalars $\alpha$ and $\beta$ such that
  $s_1 = \eps(\alpha + \beta\qi)$. The condition
  $n_1\cj{s_1} + s_1\cj{n_1} = 0$ implies $\alpha = \ci \beta$ so that
  $[s_1] = [\eps(\ci + \qi)] = [m_1n_1] = [n_1m_1]$ and $U$ is indeed a C space.
\end{proof}

\section{Extended Kinematic Mapping and Straight Lines}
\label{sec:example}

In this section we illustrate the usefulness of \autoref{th:4} at hand
of an example. In Pfurner et al. \cite{pfurner16} the authors consider
the map \eqref{eq:1} but without imposing the Study condition
(``extended kinematic map''). The action of a dual quaternion on
points is still a rigid body transformation but the map is no longer
an injection. Pfurner et al. (2016) show that the fibers of
\eqref{eq:1} (the set of pre-images of a fixed rigid body
displacement) is a straight line. More precisely, the fiber incident
with a point $[p] = [p' + \eps p''] \in P^7 \setminus \EG$ is the
straight line $[p' + \eps p''] \vee [\eps p''] = \fiberproj([p])$.
This line intersects the Study quadric in $[\eps p]$ and one further
point which gives the corresponding displacement in the classical
kinematic map (which also the underlying concept in this paper).

Pfurner et al. (2016) also mention that the motion corresponding to a
generic straight line in $P^7$ is a Darboux motion
\cite{bottema90,li15}, that is, the motion obtained by composing a
planar elliptic motion with a harmonic oscillation in orthogonal
direction. This is fairly straightforward to see. By virtue of
\eqref{eq:1}, all trajectories are rational curves degree two or less
which already limits possible candidates to Darboux motions, rotations
with fixed axis, and curvilinear translations along rational curves of
degree two or less.

We will now prove that the motion of a generic straight line $\ell$ is
actually a \emph{vertical} Darboux motion. In this limiting case of
the generic Darboux motion the elliptic translation is replaced by a
rotation with fixed axis. Our proof is based on the fact that the
three-space spanned by $\ell$ and its fiber $\fiberproj(\ell)$ is a C
space.

\begin{theorem}
  \label{th:5}
  If a straight line $\ell \subset P^7$ contains a point
  $[p] \notin \EG$ and is neither contained in the null cone $\NC$ nor
  in the four-dimensional subspace $\TS \coloneqq [p] \vee \EG$, the
  span of $\ell$ and $\fiberproj(\ell)$ is a C space.
\end{theorem}

\begin{proof}
  Without loss of generality we assume $[p] = [1]$. Then $\TS$ is the
  space of all translations. The straight line $\ell$ is neither
  contained in nor tangent to the null cone $\NC$ (because it is real
  and does not intersect $\EG$). Hence, it intersects $\NC$ in two
  points $[a] = [a' + \eps a'']$, $[b] = [b' + \eps b'']$. They
  satisfy the null cone condition $a'\cj{a'} = b'\cj{b'} = 0$. Because
  $[1] \in [a] \vee [b]$, we may write $b = \cj{a'} - \eps a''$. Note
  that $a$ and $b$ are complex conjugates but we will not make direct
  use of this, thus avoiding possible confusion arising from mixing of
  complex and quaternion conjugation.

  If the vectors $a'$ and $\cj{a'}$ are linearly dependent, there
  exists a linear combination $\eps c''$ of $a$ and $b$ with
  $c'' \neq 0$. But then $[1] \neq [\eps c''] \in \TS$ and, contrary
  to our assumption, $\ell = [1] \vee [\eps c''] \subset \TS$. Hence,
  $a'$ and $\cj{a'}$ are linearly independent. This implies that the
  points $[a] = [a'+\eps a'']$, $[b] = [\cj{a'} - \eps a'']$,
  $\varphi([a]) = [\eps a']$, and $\varphi([b]) = [\eps \cj{a'}]$ span
  a three-dimensional space $U$.

  Now we compute the intersection of $U$ with $\SQ$ and $\NC$. We set
  $x = \alpha a + \beta \cj{a} + \eps(\gamma a' + \delta \cj{a'})$ and
  solve the equation $x\cj{x} = 0$ for $[\alpha,\beta,\gamma,\delta]$.
  The left-hand side is a dual number. Equating its primal part with
  zero yields
  \begin{equation*}
    \alpha\beta({a'}^2 + \cj{a'}^2) = 0.
  \end{equation*}
  Note that $f \coloneqq {a'}^2 + \cj{a'}^2$ is a real number and it
  is different from zero because of linear independence of $a'$ and
  $\cj{a'}$. Hence $\alpha = 0$ or $\beta = 0$. We further abbreviate
  \begin{equation*}
    g_1 \coloneqq a'\cj{a''}+a''\cj{a'},\quad
    g_2 \coloneqq \cj{a'}\,\cj{a''}+a''a'
  \end{equation*}
  (these are real numbers as well) and equate the dual part of
  $x\cj{x} = 0$ with zero:
  \begin{equation*}
    (\alpha\delta+\beta\gamma)f +
    \alpha(\alpha-\beta)g_1 +
    \beta(\alpha-\beta)g_2 = 0
  \end{equation*}
  In case of $\beta = 0$, this boils down to
  \begin{equation*}
    \alpha(\delta f + \alpha g_1) = 0
  \end{equation*}
  while we get
  \begin{equation*}
    \beta(\gamma f - \beta g_2) = 0
  \end{equation*}
  if $\alpha = 0$. From these considerations we infer that
  $U \cap \SQ \cap \NC$ consists of precisely three straight lines:
  The line $e_1$ given by
  $[\alpha,\beta,\gamma,\delta] = [0,0,\gamma,\delta]$ which is
  contained in the exceptional generator $\EG$, and the two lines
  $\ell_1$, $\ell_2$ given by
  \begin{equation*}
    [\alpha,\beta,\gamma,\delta] = [-f,0,\gamma,g_1],
    \quad\text{and}\quad
    [\alpha,\beta,\gamma,\delta] = [0,f,g_2,\delta]
  \end{equation*}
  respectively. Because of $f \neq 0$, the lines $\ell_1$ and $\ell_2$
  do not intersect and they are complex conjugates by construction.
  This already implies that $U \cap \SQ$ is a ruled quadric but
  neither a quadratic cone nor a double plane (but we still have to
  argue that it is not a pair of planes).

  \begin{figure}
    \centering
    \includegraphics{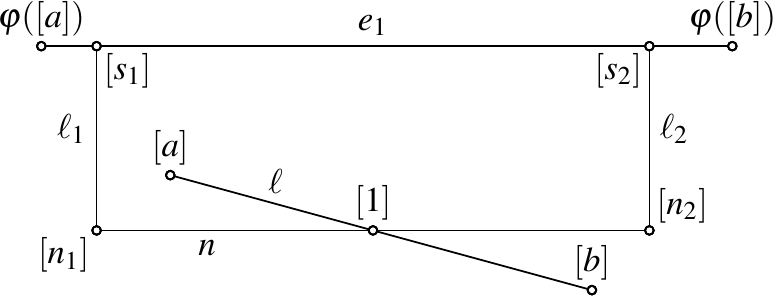}
    \caption{Lines in the proof of \autoref{th:5}.}
    \label{fig:darboux-space}
  \end{figure}

  For $\alpha = \beta = 0$ and $\gamma = \infty$ or $\delta = \infty$
  we obtain the intersection points
  \begin{equation}
    \label{eq:10}
    [s_1] = [\eps a']
    \quad\text{and}\quad
    [s_2] = [\eps \cj{a'}]
  \end{equation}
  of $\ell_1$ and $\ell_2$, respectively, with $e_1$. There is a unique
  line $n$ through $[1]$ that intersects both $\ell_1$ and $\ell_2$.
  Its intersection points $[n_1]$, $[n_2]$ with $\ell_1$ and $\ell_2$,
  respectively, are given by
  \begin{equation*}
    [\alpha,\beta,\gamma,\delta] = [-f,0,g_2,g_1]
    \quad\text{and}\quad
    [\alpha,\beta,\gamma,\delta] = [0,f,g_2,g_1].
  \end{equation*}
  The line $[n_1] \vee [n_2]$ is a fourth ruling in the intersection
  of $U$ and $\SQ$ and does not intersect $[e_1]$. Hence $U$ and $\SQ$
  cannot intersect in a pair of planes.

  The quadric $\QQ = \SQ \cap U$ is given by the equations
  $x = \alpha a + \beta \cj{a} + \eps(\gamma a' + \delta \cj{a'})$
  where the parameters $\alpha$, $\beta$, $\gamma$, and $\delta$
  satisfy $\alpha\delta + \beta\gamma = 0$. From this we see that that
  its fiber projection is given by $[\eps(\alpha a' + \beta\cj{a'})]$.
  It, indeed, coincides with $e_1$ so that also the last condition of
  \autoref{th:4} is fulfilled.
\end{proof}

\begin{corollary}
  \label{cor:1}
  A straight line $\ell$ satisfying the assumptions of \autoref{th:5}
  corresponds, via \eqref{eq:1}, to a \emph{vertical} Darboux motion.
\end{corollary}

\begin{proof}
  Since the map \eqref{eq:1} generically doubles the degree, the
  trajectories of the motion in question are at most quadratic and the
  motion itself is either a curvilinear translation along a curve of
  degree two or less, the rotation about a fixed axis or a Darboux
  motion. Pure translations are excluded by the assumptions of
  \autoref{th:5}, non-vertical Darboux motions by its conclusion.
\end{proof}

\begin{remark}
  \autoref{cor:1} is also valid for straight lines contained in the
  Study quadric $\SQ$ unless they intersect the exceptional generator
  $\EG$. These lines are known to describe rotations about a fixed
  axis (Selig 2005, Section~11.2) \cite{selig05}, that is, a vertical
  Darboux motion with zero amplitude. In this sense, we may say that
  \emph{all straight lines in extended kinematic image space that
    intersect the null cone in two distinct points describe a vertical
    Darboux motion.}
\end{remark}

\section{Conclusion}
\label{sec:conclusion}

We provided necessary and sufficient geometric conditions on the
transformation group induced by coordinate changes in the fixed and
moving frame and on the kinematic image of 2R, RP, and PR dyads. While
many properties of these objects are already known, sufficiency of
certain collections of conditions seems to be new. We payed proper
attention to the case of commuting joints (which is trivial for RR
dyads), highlighted the role of left and right rulings, and introduced
the fiber projectivity. It is important in our characterisation of RP
and PR spaces but also for the kinematic images of Darboux motions, as
illustrated in \autoref{ex:1} and \autoref{sec:example}.

\appendix
\section{Auxiliary results}
\label{sec:appendix}

In this appendix we prove two technical results of purely projective
nature and with no direct relations to kinematics. The formulation of
\autoref{lem:6} could be simplified but in its present form its
application in the proof of \autoref{lem:7} is apparent.

\begin{lemma}
  \label{lem:6}
  Given a three-space $E \subset P^7$ and four points $[u'_1]$,
  $[v'_1]$, $[u'_2]$, $[v'_2]$ that span a three-space $F \subset P^7$
  with $E \cap F = \varnothing$, set $U_1 \coloneqq [u'_1] \vee E$,
  $V_1 \coloneqq [v'_1] \vee E$, $U_2 \coloneqq [u'_2] \vee E$,
  $V_2 \coloneqq [v'_2] \vee E$ and consider four projections
  \begin{equation*}
    \zeta_1\colon U_1 \to V_1,\quad
    \eta_1\colon V_1 \to U_2,\quad
    \zeta_2\colon U_2 \to V_2,\quad
    \eta_2\colon V_2 \to U_1
  \end{equation*}
  with respective centres $[m_1]$, $[n_1]$, $[m_2]$, and $[n_2]$. If
  the centres are pairwise different and span a plane $L$ that is
  complementary to $U_1$, $V_1$, $U_2$, and $V_2$, respectively, the
  composition $\eta_2 \circ \zeta_2 \circ \eta_1 \circ \zeta_1$ of the
  four projections is the identity.
\end{lemma}

\begin{figure}
  \centering
  \includegraphics{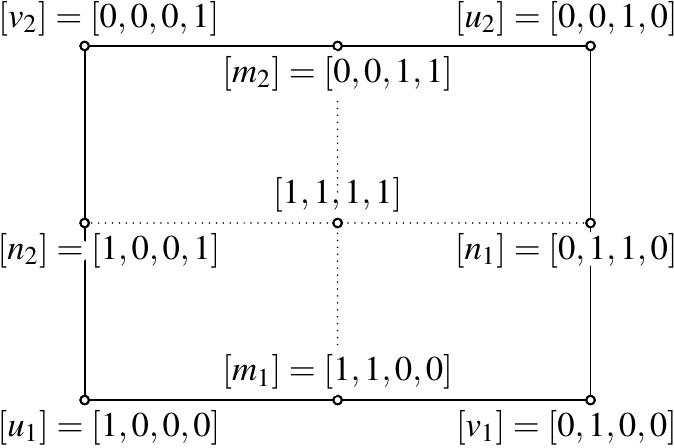}
  \caption{Projective coordinates in the proof of \autoref{lem:6}
    (four trailing zeros are omitted).}
  \label{fig:procosy}
\end{figure}

\begin{proof}
  Note that $\dim U_1 \vee V_1 = 5$ so that the projection $\zeta_1$
  (and also $\eta_1$, $\zeta_2$, $\eta_2$) from a suitable point is
  well defined. Take an arbitrary point $[u_1] \in U_1$ but not in $E$
  and set $[v_1] \coloneqq \zeta_1([u_1])$,
  $[u_2] \coloneqq \eta_1([v_1])$, $[v_2] \coloneqq \zeta_2([u_2])$.
  Then the projective space
  $G \coloneqq [u_1] \vee [v_1] \vee [u_2] \vee [v_2] = [u_1] \vee
  [m_1] \vee [n_1] \vee [m_2] \vee [n_2]$ (\autoref{fig:procosy}) is
  of dimension three and does not intersect $E$ because of the
  dimension formula
  \begin{equation*}
    \begin{aligned}
      \dim (G \cap E) &= \dim G + \dim E - \dim (G \vee E) \\
                      &= 3 + 3 - \dim ([u_1] \vee L \vee E) = 3 + 3 - 7 = -1.
    \end{aligned}
  \end{equation*}
  Therefore, we may take $[u_1]$, $[v_1]$, $[u_2]$ and $[v_2]$ in that
  order as base points of a projective coordinate system in $G$ and
  complement them by four further base points in $E$ and a suitable
  unit point to a projective coordinate system of $P^7$. We select the
  unit point such that its projection into $G$ from $E$ gives
  $([m_1] \vee [m_2]) \cap ([n_1] \vee [n_2])$. Then, the projection
  centres are
  \begin{gather*}
    [m_1] = [1,1,0,0,0,0,0,0],\quad
    [n_1] = [0,1,1,0,0,0,0,0],\\
    [m_2] = [0,0,1,1,0,0,0,0],\quad
    [n_2] = [1,0,0,1,0,0,0,0].
  \end{gather*}
  We pick an arbitrary point
  $h_2 \coloneqq [x_0,0,0,0,x_4,x_5,x_6,x_7] \in U_1$ and compute
  \begin{equation*}
    \begin{aligned}
      \zeta_1(h_2) &= [0,-x_0,0,0,x_4,x_5,x_6,x_7] \eqqcolon z_1,\\
      \eta_1(z_1) &= [0,0,x_0,0,x_4,x_5,x_6,x_7] \eqqcolon h_1,\\
      \zeta_2(h_1) &= [0,0,0,-x_0,x_4,x_5,x_6,x_7] \eqqcolon z_2,\\
      \eta_2(z_2) &= [x_0,0,0,0,x_4,x_5,x_6,x_7] = h_2.
    \end{aligned}
  \end{equation*}
  This concludes the proof.
\end{proof}

\begin{lemma}
  \label{lem:7}
  Let $E$, $F \subset P^7$ be non-intersecting three-spaces, the
  latter spanned by points $[u'_1]$, $[v'_1]$, $[u'_2]$, $[v'_2]$.
  Consider further a regular quadric $\QQ$ containing $E$ and four
  vertices $[m_1]$, $[n_1]$, $[m_2]$, $[n_2] \in \QQ$ of a 
  quadrilateral in a plane $L$ that is complementary to
  \begin{equation*}
    E \vee [u'_1],\quad
    E \vee [v'_1],\quad
    E \vee [u'_2],\quad\text{and}\quad
    E \vee [v'_2].
  \end{equation*}
  Then there exists a unique spatial quadrilateral with vertices
  $[u_1]$, $[v_1]$, $[u_2]$, $[v_2]$ that is contained in $\QQ$, whose
  sides $[u_1] \vee [v_1]$, $[v_1] \vee [u_2]$, $[u_2] \vee [v_2]$,
  $[v_2] \vee [u_1]$ are, in that order, incident with $[m_1]$,
  $[n_1]$, $[m_2]$, $[n_2]$ and whose projection from $E$ into $F$ is
  the quadrilateral with vertices $[u'_1]$, $[v'_1]$, $[u'_2]$,
  and~$[v'_2]$.
\end{lemma}

\begin{proof}
  Denote the quadratic form associated to $\QQ$ by $\qf$. Given
  $[u'_1]$, $[v'_1]$, $[u'_2]$, $[v'_2]$, we have to reconstruct
  $[u_1]$, $[v_1]$, $[u_2]$, $[v_2]$ subject to the constraints
  \begin{equation}
    \label{eq:12}
    \qf(u_i,u_i) = \qf(v_i,v_i) = \qf(u_i,v_j) = 0, \quad i,j \in \{1,2\}
  \end{equation}
  and
  \begin{equation}
    \label{eq:13}
    \begin{gathered}
      [m_1] \in [u_1] \vee [v_1],\quad
      [n_1] \in [v_1] \vee [u_2],\\
      [m_2] \in [u_2] \vee [v_2],\quad
      [n_2] \in [v_2] \vee [u_1].
    \end{gathered}
  \end{equation}
  If $[u_1]$ is given, we find $[v_1]$ by projecting $[u_1]$ from
  centre $[m_1]$ onto $E \vee [v'_1]$. \autoref{lem:6} tells us that
  we can find $[u_2]$ and $[v_2]$ in similar manner such that
  \eqref{eq:13} is satisfied. Then, some of the conditions in
  \eqref{eq:12} become redundant and it is sufficient to consider only
  \begin{equation}
    \label{eq:14}
    \qf(u_1,u_1) = \qf(v_1,v_1) = \qf(u_2,u_2) = \qf(v_2,v_2) = 0.
  \end{equation}
  In a projective coordinate system with base points $[u'_1]$,
  $[v'_1]$, $[u'_2]$, $[v'_2] \in F$ and further base points in $E$,
  the quadratic form $\qf$ of the quadric $\QQ$ is described by a
  matrix of the shape
  \begin{equation*}
    \begin{bmatrix}
      A & B \\
      B & O
    \end{bmatrix}
  \end{equation*}
  where $A$, $B$ and $O$ are matrices of dimension $4 \times 4$, $O$
  is the zero matrix and $B$ is regular. Now \eqref{eq:14} gives rise
  to a linear system for the unknown coordinates of $[u_1]$, $[v_1]$,
  $[u_2]$, and $[v_2]$. We have
  \begin{equation*}
    \begin{aligned}
      v_1 &= u_1 + \zeta_1 m_1,\\
      u_2 &= v_1 + \eta_1 n_1 = u_1 + \zeta_1 m_1 + \eta_1 n_1,\\
      v_2 &= u_2 + \zeta_2 m_2 = u_1 + \zeta_1 m_1 + \eta_1 n_1 + \zeta_2 m_2
    \end{aligned}
  \end{equation*}
  with certain scalars $\zeta_1$, $\eta_1$, $\zeta_2$ that can be
  computed from $u'_1$, $v'_1$, $u'_2$, $v'_2$ and $m_1$, $n_1$,
  $m_2$, $n_2$ alone. Now we write $[u_1] = [1,0,0,0,x_0,x_1x_2,x_3]$
  with unknown $x_0$, $x_1$, $x_2$, $x_3$. Then, \eqref{eq:14} yields
  \begin{equation}
    \label{eq:15}
    \begin{aligned}
      0 &= [1,0,0,0] \cdot B [x_0,x_1,x_2,x_3]^\tp,\\
      0 &= [0,1,0,0] \cdot B ([x_0,x_1,x_2,x_3] + \zeta_1m''_1)^\tp,\\
      0 &= [0,0,1,0] \cdot B ([x_0,x_1,x_2,x_3] + \zeta_1m''_1 + \eta_1 n''_1)^\tp,\\
      0 &= [0,0,0,1] \cdot B ([x_0,x_1,x_2,x_3] + \zeta_1m''_1 +
      \eta_1 n''_1 + \zeta_2m''_2)^\tp
    \end{aligned}
  \end{equation}
  where the double prime denotes projection on the last four
  coordinates. \eqref{eq:15} is a linear system for $x_0$, $x_1$,
  $x_2$, $x_3$ with regular coefficient matrix $B$. Hence, the
  solution is unique.
\end{proof}

\section*{Acknowledgement}
\label{sec:acknowledgment}

This work was supported by the Austrian Science Fund (FWF): P\;26607
(Algebraic Methods in Kinematics: Motion Factorisation and Bond
Theory).



\end{document}